\documentclass[english,10pt,journal]{IEEEtran}
\usepackage[T1]{fontenc}
\usepackage[latin9]{inputenc}
\usepackage{algorithm}
\usepackage{algorithmic}
\usepackage{fancybox}
\usepackage{calc}
\usepackage{units}
\usepackage{amsthm,amscd,mathrsfs}
\usepackage{amsmath}
\usepackage{amssymb}
\usepackage{graphicx}
\usepackage{comment}
\usepackage{todonotes}
\usepackage{soul}
\usepackage{balance}
\usepackage{framed}
\usepackage[normalem]{ulem}

\newcommand{\rotatedd}{{\mathpalette\rotd\relax}}
\newcommand{\rotd}[2]{\rotatebox[origin=c]{190}{$#1d$}}

\makeatletter
\theoremstyle{plain}
\newtheorem{theorem}{\protect\theoremname}
\newtheorem{corollary}{Corollary}
\newtheorem{lemma}{Lemma}
\newtheorem{example}{Example}
\theoremstyle{plain}
\newtheorem{proposition}{\protect\propositionname}
\newtheorem{definition}{Definition}
\@ifundefined{showcaptionsetup}{}{%
 \PassOptionsToPackage{caption=false}{subfig}}
\usepackage{subfig}
\makeatother

\usepackage{babel}
\providecommand{\propositionname}{Proposition}
\providecommand{\theoremname}{Theorem}

\begin{document}

\title{Distributed Algorithms for Feature Extraction Off-loading in Multi-Camera Visual Sensor Networks}

\author{Emil Eriksson, Gy\"{o}rgy D\'{a}n, Viktoria Fodor \\ School of Electrical Engineering, KTH Royal Institute of Technology, Stockholm, Sweden\\\{emieri,gyuri,vfodor\}@kth.se}
\maketitle

\begin{abstract}
Real-time visual analysis tasks, like tracking and recognition, require swift execution of computationally intensive algorithms.
Visual sensor networks can be enabled to perform such tasks by augmenting the sensor network with processing nodes and distributing the computational burden in a way that the cameras contend for the processing nodes while trying to minimize their task completion times.
In this paper, we formulate the problem of minimizing the completion time of all camera sensors as an optimization problem.
We propose algorithms for fully distributed optimization, analyze the existence of equilibrium allocations, evaluate
the effect of the network topology and of the video characteristics, and the benefits of central coordination.
Our results demonstrate that with sufficient information available, distributed optimization can provide low completion times, moreover predictable and stable performance can be achieved with additional, sparse central coordination.
\end{abstract}


\let\thefootnote\relax\footnote{This work was partly funded by the EU FP7 FET GreenEyes project (296676) and the Modane project funded by SSF.}

\vspace{-4mm}
\section{Introduction}
\label{sec::introduction}
Many real-time computer vision applications, like surveillance, tracking, traffic monitoring and augmented reality require the timely processing of visual information~\cite{Muller2011,rana3dtv14,liemPR11,Zhou2009CVIU,Ayazoglu:2011:DSC:2355573.2356391,helmerICRA10}.
If, in addition, visual information from multiple cameras is available, the application precision can be increased~\cite{naikal2010towards}, and events can be reconstructed in 3D~\cite{Muller:2005:RDE:2322564.2323726}.
With the emergence of cheap cameras and network devices, visual sensor networks (VSNs) could, in principle, enable wide-spread deployment of these popular applications, but in practice visual processing in VSNs faces two challenges.
On the one hand, the high computational complexity of the image processing tasks, paired with the limitations of the sensor nodes, prevents the processing to be performed locally by the camera equipped devices.
On the other hand, considering the delay limit and the energy resources of the network nodes, the large amount of pixel data in the frames makes it infeasible to transmit all data through the sensor network to a central processing node. 
Bandwidth requirements may be reduced by using image downsampling, reducing the color depth, or other video encoding techniques. 
However, any lossy compression may also affect the results of visual processing tasks~\cite{Duan2012}.

A promising solution to overcome these challenges is to augment the sensor network with processing nodes that have suitable memory and computational capacity, and to perform the image processing at these nodes.
By assigning each processing node a part of the complete frame, for example by dividing each frame into multiple sub-areas~\cite{KhanTMM15}, frames can be processed in a distributed manner.
In contrast to the camera nodes, the processing nodes do not need to be calibrated and can be installed or replaced with ease, possibly extending the lifetime of the camera equipped nodes.
Reduced maintenance and extended lifetime are particularly important for VSNs deployed in remote or hazardous areas, or in animal habitats.
As multiple sensors may need to share the processing nodes as well as the wireless channel, optimization of the distribution of the processing tasks is non-trivial.

In this paper we consider the case of local visual feature-based visual analysis~\cite{Zhou2009CVIU,Ayazoglu:2011:DSC:2355573.2356391,Rosten2010,Leuten2011}, where the visual analysis application utilizes the features extracted from the frames captured by multiple sensors. 
The sensors can leverage the capabilities of the processing nodes for detecting and extracting the feature descriptors, and aim at minimizing the time until the extraction of all features from all frames is completed. 
To perform this minimization, the sensors can decide the set of processing nodes used, the schedule of the transmission of the pixel information through the shared wireless channel, and the size of the frame sub-areas sent to the processing nodes.
Due to the dynamic visual content of the frame the sensors need to revise the processing allocation for each video frame.

We develop an analytical model of the system and formulate completion time minimization as an optimization problem for the case when all system parameters are known and the optimization can be performed centrally, and we show that the problem is NP-hard. We propose a centralized
approximation and fully distributed optimization algorithms where the sensors use only locally available information obtained via measurements, as well as based on additional information received through signaling between the processing nodes and the sensors. 
We provide sufficient conditions for the existence of equilibrium allocations for the distributed algorithms, and analyze the convergence properties of the distributed algorithms under synchronous and asynchronous revisions.
Then, we consider the case when the distributed algorithms are supported by periodical centralized coordination.
We use these to investigate how the amount of information available at the sensor nodes affects the achievable completion time.
We use simulations to give insight into the convergence properties and the performance of the algorithms under various VSN topologies and video characteristics. 
Our results demonstrate that distributed optimization can provide low completion times already with limited amount of shared information, moreover predictable and stable performance can be achieved with additional, sparse central coordination.

The rest of the paper is organized as follows.
In Section~\ref{sec::related-work} we review related work.
In Section~\ref{sec::system} we describe the considered system and in Section~\ref{sec::problem} we formulate the problem of completion time minimization.
In Section~\ref{seq::algorithms} we present and analyze fully distributed algorithms for solving the completion time minimization problem, and in Section~\ref{sec::centralized} we introduce centralized and coordinated algorithms.
In Section~\ref{sec::numerical} we present numerical results and we conclude the paper in Section~\ref{sec::conclusion}.

\section{Related work}
\label{sec::related-work}
Visual analysis applications utilizing many camera nodes are discussed among others in~\cite{Muller2011,rana3dtv14} for free viewpoint television, in~\cite{liemPR11,Zhou2009CVIU,Ayazoglu:2011:DSC:2355573.2356391} for localization and tracking and in~\cite{helmerICRA10} for high accuracy object recognition.
The challenge of visual analysis at nodes with limited processing power is addressed in~\cite{Rosten2010,Leuten2011}, defining feature extraction schemes with low computational complexity.
To decrease the transmission bandwidth requirements of pixel information,~\cite{Duan2012,Chao2013} propose lossy image coding schemes optimized for descriptor extraction, while~\cite{Chandrasekhar2010,Jegou2011,Redondi2013} give solutions to decrease the number and the size of the descriptors to be transmitted.
Considering video sequences with temporal correlation,~\cite{Ta2009} limits the frame areas of interest, while~\cite{SullivanTCS12,BaraffioICIP13} proposes intra- and inter-frame coding for the descriptors.
However, results in~\cite{RedondiDSP13} show that even under optimized extraction and coding, the processing at the camera sensor or at the sink node of the VSN leads to significant delay, which motivates the introduction of in-network processing in VSNs~\cite{EDF2014DCOSS,BaroffioCCRTDEFAM2014ICIP,EDF2016TMC,redondi2015cooperative,redondi2015mathematical}.

Optimal load scheduling for distributed systems is addressed in~\cite{BharadwajCC2003}, in the framework of Divisible Load Theory (DLT), with the general result that minimum completion time is achieved, if all processors finish the processing at the same time.
In~\cite{mequanint2006wireless,li2007sensing} DLT is used in the context of wireless sensor networks with single and multi level tree network topologies. 
Usually three decisions need to be made: the subset of the processors used, the order they receive their share of workload, and the division of the workload.
Unfortunately, the results are specific to a given system setup, and therefore scheduling solutions are derived for given, simplified systems with simple topologies~\cite{BharadwajTPDS1994,BharadwajTPDS2000}.
In~\cite{EDF2016TMC} we show that the application of DLT for distributed visual processing is non-trivial even in the case of a single sensor node, due to the transmission overhead introduced by distributed feature extraction, and due to the dynamism of the frame content in the video. We introduce the distributed, multiple sensor, multiple processing node case in~\cite{EDF2015NW,EPD2015ICMEW}, and derive basic convergence result. In this paper we provide a rigorous evaluation of distributed solutions as well as solutions with central coordination support, and evaluate the effect of the network topology as well as the effect of the temporal dynamism of the video content.
Unlike~\cite{EDF2015NW,EPD2015ICMEW}, in this paper we consider the realistic constraint that a processing node can not start processing a slice until it finishes receiving the data for the slice, prove that the completion time minimization problem is NP-hard, and propose an efficent approximation based on nearest neighbor search.

Related to our work is the problem of learning in non-cooperative games~\cite{Monderer96,Monderer96b}.
Studies of learning in games usually consider models of perfect information for the analysis of convergence~\cite{Monderer96,Monderer96b,Pacifici12jsac,Berger07}.
Recent works on experimentation dynamics and regret testing models consider that players can only observe their own payoffs~\cite{Foster06,Cigler2011,Pradelski2012}, but these learning models provide asymptotic convergence guarantees, and thus convergence is prohibitively slow.
In this paper we consider two models of imperfect information, and provide equilibrium existence and convergence results.
Our results also highlight the potential trade-off between predictable and good performance in learning. 

\section{System model}
\label{sec::system}
We consider a visual sensor network (VSN) that consists of a set of sensor nodes $\mathcal{S}$, $|\mathcal{S}|=S$, a set of processing nodes $\mathcal{N}$, $|\mathcal{N}|=N$,
and a central coordinator with significant computational power.
Sensor nodes $s \in \mathcal{S}$ capture periodically a sequence $\mathcal{I}_s=\{1,...\}$ of frames of width $w$ pixels. 
We call the set of $S$ frames captured simultaneously by the sensor nodes a \emph{multi-view} frame.
The objective of the system is to process the multi-view frames in the shortest possible time.

For the delegation of the computation, sensor node $s$ divides frame $i$ into $V_s^i\leq N$ vertical slices.
This scheme was referred to as area-split in~\cite{KhanTMM15,KhanDSP13}.
We define slice $v$ using its normalized leftmost and rightmost horizontal coordinates, $x_{s,v-1}^{i}$ and $x_{s,v}^{i}$, i.e., $x_{s,0}^{i}=0$ and $x_{s,V_s^i}^{i}=1$,
and we define the cutpoint location vector for frame $i$ as $x_{s}^{i}=\{x_{s,0}^{i}, \dots, x_{s,V_s^i}^{i}\}$. 
We call the vector of cutpoint location vectors the \emph{cutpoint location profile} and denote it by ${\bf x^i} = \left\lbrace x_{1}^{i},\ldots,x_{s}^{i} \right\rbrace$.
For convenience, we use $y_{s,v}^{i}=x_{s,v}^{i}-x_{s,v-1}^{i}$ to denote the normalized width of slice $v$, and
we define $y_{s,v}^{i}=0$ for $v\leq 0$ and for $v>V_s^i$. Thus, by definition, $\sum_{v=1}^{V_s^i}{y_{s,v}^{i}}=1$.
Sensor $s$ transmits slice $1\leq v\leq V_s^i$ to processing node $d_s^i(v)\in\mathcal{N}$ for processing. We define $d_s^i$ as a sequence with $V_s^i$ distinct elements,
and with slight abuse of notation we use $n\in d_s^i$ if $d_s^i(v)=n$ for some $1\leq v \leq V_s^i$, i.e., node $n$ is used by sensor $s$. Thus, $d_s^i$ defines a
partial permutation of $\{1,\ldots,S\}$. We use
the notation $\rotatedd_s^{i}(n)$ for the slice that sensor $s$ assigns to node $n$.
That is, we refer to $d_s^i$ as the assignment by sensor $s$, and to $\rotatedd_s^{i}$ as its inverse.
Furthermore, we refer to ${\bf d_{s}^{i}} = \left\lbrace d_1^i,\ldots,d_s^i \right\rbrace$ as the \emph{assignment profile}.

\subsection{Visual feature extraction}
Each processing node $n$ computes local visual features from the slices assigned to it. A processing node starts to process a slice as soon as it is completely received, and performs parallel processing of multiple slices. In applications where features extracted from frames captured by multiple cameras are needed, e.g., in the case of multi-camera tracking for updating a hidden-Markov model or a particle filter, the extraction of the features from all frames should finish at the same time. We thus consider that if a processing node $n$ has to process slices from different sensors, then it allocates its processing power such that the processing of all slices is completed at the same time.

The computation of local features starts with interest point detection, by applying a blob detector
or an edge detector at every pixel of the slice~\cite{Leuten2011,Bay2008,Calonder2010}.
For each pixel, the detector computes a response score based on a square area centered around it.
We denote the side length of the square normalized by the width of the frame by $2o$. The side length $2ow$ of the square (in pixels) depends on the applied detector.

A pixel is identified as an interest point if the response score exceeds the detection
threshold $\vartheta\in\boldsymbol{\Theta}\subseteq\mathbb{R}^{+}$. We denote the distribution of interest points of frame $i$ at sensor $s$ as $f_{s}^{i}$, with CDF of $F_{s}^{i}$ .
The time it takes to detect interest points can be modeled as a linear function of the frame size in pixels
and of the number 
$\xi_{s,v}^{i}=F_{s}^{i}(x_{s,v}^{i})-F_s^i(x_{s,v-1}^{i})$ 
of detected interest points; this model was validated on a BeagleBone Black single board computer in~\cite{EDF2014DCOSS,EDF2016TMC}
and on an Intel Imote2 platform in~\cite{RedondiDSP13}. We can thus model the detection time for slice $v$ from sensor $s$ at processing node $n$ as a function of the slice width $y_{s,v}^{i}$ and of the number $\xi_{s,v}^{i}$ of detected interest points as an affine function $P_n(y_{s,v}^{i}+\alpha_f \xi_{s,v}^{i})$, where $P_n$ is the per unit processing time of node $n$ and $\alpha_f$ is a normalization constant.

After detection, a feature descriptor is extracted for each interest point by comparing pixel intensities.
The time it takes to extract the descriptors can be modeled as a linear function of the number $\xi_{s,v}^{i}$ of
interest points detected, as shown in~\cite{EDF2014DCOSS}. We can thus model the detection and extraction time as $P_n(y_{s,v}^{i}+(\alpha_f+\alpha_e) \xi_{s,v}^{i})=P_n(y_{s,v}^{i}+\alpha_d \xi_{s,v}^{i})$.
We consider that normalization constant $\alpha_f$, $\alpha_e$ and thus $\alpha_d$ are the same for all processing nodes, which is reasonable if the nodes have a similar computing architecture (e.g., instruction set). Due to parallel processing, $P_n$ depends on the number of slices that processing node $n$ is processing in parallel.

\subsection{Communication Model}
The nodes communicate using a wireless communication protocol, such as IEEE 802.15.4 or IEEE 802.11. Sensors transmit to one processing node at
a time, while processing nodes can receive MAC packets from many sensors, through the shared channel. Transmissions suffer from packet losses due to wireless channel impairments.
As measurement studies show, the loss burst lengths at the receiver have low mean and variance in the order of a couple of MAC frames~\cite{LacanWIOPT2006,HarwellVTC2004,GuhaTVT2008}.
Therefore, a widely used model of the loss process is a low-order Markov-chain, with fast decaying correlation and short mixing time.
In the system we consider, the amount of data to be transmitted to the processing nodes is relatively large, and therefore it is reasonable to model the average transmission time from sensor $s$ to processing node $n$, including the retransmissions, as a linear function of the amount of transmitted data.
We denote the transmission time coefficient by $C_{s,n}$, which can be interpreted as the average per frame transmission time.
As the throughput is close to stationary over short timescales, $C_{s,n}$ can be estimated~\cite{Petrova2006wcnc}.
When there are several sensors transmitting data, the MAC protocol provides airtime fairness for the transmitters~\cite{Joshi2008TMC}, thus the actual transmission time coefficient is proportional to the number of sensors transmitting.
For example, when $S$ sensors are transmitting, the actual coefficient is $S C_{s,n}$.

Recall that interest point detection involves applying a square filter of size $2o$ at each pixel.
Thus, for correct operation, each slice $v$ has to be appended by an overlap area of width $o$ on one or both sides.
The resulting regions of overlap in adjacent slices could in principle be transmitted in multicast to the appropriate processing nodes, but experimental results show that multicast transmission suffers from low throughput in practice due to lack of link layer retransmissions and missing channel quality information~\cite{Kostuch2009Wmnc}, we thus consider that all data transmissions are done using unicast.

We consider that the transmission time coefficient between the sensors and the central coordinator is high enough so that
transmitting pixel information to the central coordinator is infeasible, but the sensor nodes, the processing nodes, and the coordinator
can exchange control information with a delay that is negligible. 

\begin{figure}[t]
\includegraphics[width=1\columnwidth]{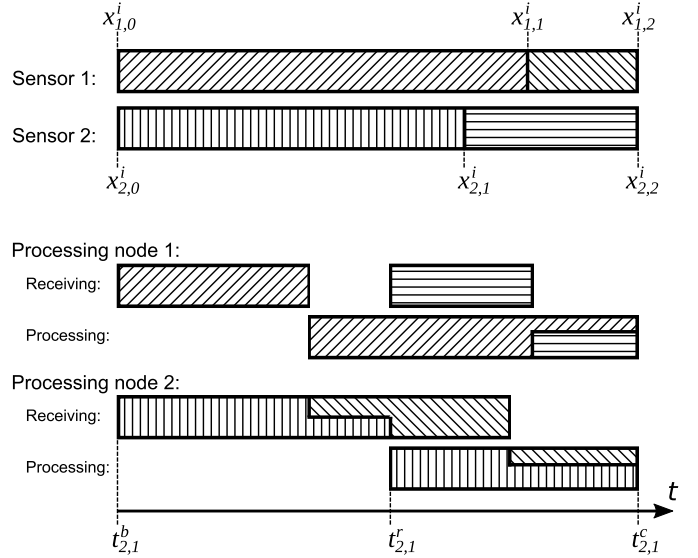}
\caption{Transmission and processing in a two sensor, two processing node system. The pattern identifies the slice being transmitted or processed by the nodes. Where two patterns overlap in time at the same node, the height of the pattern is proportional to the fraction of transmission or processing resources assigned to the slice. The assigment functions are $d_1^i = \left\lbrace 1,2 \right\rbrace, \, d_2^i = \left\lbrace 2,1 \right\rbrace$. }
\label{fig::completion-illustration}
\vspace{-6mm}
\end{figure}

\section{Completion Time and Problem Formulation}
\label{sec::problem}

In the following we define the completion time of the system as a function of the assignment profile ${\bf d^i}$ and the cutpoint location profile ${\bf x^i}$. 
Based on the expression for system completion time we formulate the completion time minimization problem as an integer programming problem.

\subsection{Completion Time Model}
Using the model of transmission and processing above, let us consider the completion time of the processing of frame $i\in \mathcal{I}_s$ captured by sensor $s$.
Figure~\ref{fig::completion-illustration} illustrates the transmission and processing of slices from $S=2$ sensor nodes to $N=2$ processing nodes.
Let us denote by $t^b_{s,v}$ the time instant when processing node $d_{s}^i(v)$ receives the first bit of slice $v$ from sensor $s$, by $t^r_{s,v}$ the time instant when processing node $d_{s}^i(v)$ receives the last bit of slice $v$ from sensor $s$. 
Unlike in previous works~\cite{EDF2015NW,EPD2015ICMEW}, processing of slice $v$ may only start once the slice is completely received at $t^r_{s,v}$. 

Observe that the time $t^r_{s,v} - t^b_{s,v}$ it takes node $s$ to transmit slice $v$ to processing node $n$ depends on the number of
sensor nodes that are transmitting simultaneously, which
depends on the cutpoint location vectors $x_{s^\prime}^{i}$ and on the assignment functions $d_{s^\prime}^i$ of the other sensors.
To capture the dependence of the transmission time on  $(x_{s}^{i})_{s\in\mathcal{S}}$ and  $(d_s^i)_{s\in\mathcal{S}}$, we define the experienced transmission time coefficient at processing node $n=d_{s}^i(v)$ as
\begin{equation}
\tilde{C}_{s,n}({\bf x^i},{\bf d^i})=\left\{\begin{array}{lr}
\hspace{-2mm}(t^r_{s,v} - t^b_{s,v})/(y_{s,v}^i+o),  &  \hspace{-3mm} v=1,V_s^i \\ 
\hspace{-2mm}(t^r_{s,v} - t^b_{s,v})/(y_{s,v}^i+2o), &  \hspace{-3mm} 1<v<V_s^i 
\end{array} .\right.
\label{eq::node1-communication-coeff}
\end{equation}

Similarly, the time it takes processing node $n=d_{s}^i(v)$ to complete the processing of slice $v$ sent by sensor $s$ depends on whether or not the processing node has to process
slices from other sensors simultaneously. We define the experienced processing time coefficient of sensor $s$ at processing node $n$ as
\begin{equation}
\tilde{P}_{s,n}({\bf x^i},{\bf d^i})=(t^c_{s,v}- t^r_{s,v})/(y_{s,v}^i+\alpha_d\xi_{s,v}^i).
\label{eq::node1-processing-coeff}
\end{equation}

We can express the completion time of slice $v$ delegated by sensor $s$ to processing node $n=d_s^i(v)$ as a function of the experienced transmission
time coefficients and of the experienced processing time coefficients. For the first slice, i.e., $n=d_s^i(1)$, we have
\begin{equation}
T_{s,n}^i({\bf x^i},{\bf d^i}) = \tilde{C}_{s, n}({\bf x^i},{\bf d^i}) [y_{s,1}^{i}+o] + \tilde{P}_{s,n}({\bf x^i},{\bf d^i}) [y_{s,1}^{i} + \alpha_d\xi_{s,1}^i].
\label{eq::node1-completion-time}
\end{equation}
For the remaining slices, i.e., $n=d_s^i(v)$, $v>1$, the completion time depends also on the transmission times of previous slices
\begin{eqnarray}
  T_{s,n}^{i} ({\bf x^i},{\bf d^i})&=& \tilde{C}_{s, d^i_s(1)}({\bf x^i},{\bf d^i}) [y_{s,1}^{i}+o] + \tilde{C}_{s,n}({\bf x^i},{\bf d^i})2o \nonumber\\
   &+& \sum_{\nu=2}^{v-1}{\tilde{C}_{s, d^i_s(\nu)}({\bf x^i},{\bf d^i}) [y_{s,\nu}^{i}+2o]} \nonumber \\
   &+&  \tilde{P}_{s,n}({\bf x^i},{\bf d^i}) [y_{s,v}^{i}+\alpha_d\xi_{s,v}^{i}].\label{eq::node-completion-time}
\end{eqnarray}
Finally, we define the completion time of frame $i$ for sensor $s$ as the greatest completion time among all processing nodes
\begin{equation}
	T_s^i({\bf x^i},{\bf d^i})= \max_{n\in d_s^i} (T_{s,n}^{i}({\bf x^i},{\bf d^i})), 
\label{eq::sensor-completion-time}
\end{equation}
and the system completion time of frame $i$ as 
\begin{equation}
	T^i({\bf x^i},{\bf d^i})= \max_{s\in \mathcal{S}} (T_{s}^{i}({\bf x^i},{\bf d^i})). 
\label{eq::system-completion-time}
\end{equation}

\subsection{Completion Time Minimization (CTM) Problem}
Given the set of sensor nodes $\mathcal{S}$, the set of processing nodes $\mathcal{N}$,
the transmission and processing time coefficients $C_{s,n}$ and $P_n$,
we can formulate the completion time minimization (CTM) problem for a single multi-view frame $i$ as an integer programming problem
\begin{eqnarray}
              & & \lefteqn{\min_{({\bf x^i},{\bf d^i})} t}                                            \label{eq::ip-objective}\\
\textrm{s.t.} & T_s^i({\bf x^i},{\bf d^i})              &\leq \; t, \;\;\;  \forall s\in{\cal S}      \label{eq::minmax-completion}\\
              & x_{s,v-1}^i - x_{s,v}^i                 &\leq \;-o \;\;1 \leq v \leq V_s^i            \label{eq::increasing}\\
              & x_{s,v}^iw                              &\in  \;\{1,\ldots,w\} \;\;\;1 \leq v \leq V_s^i \label{eq::integer}
\end{eqnarray}
where $w$ is the width of the individual frames in pixels. Constraint \eqref{eq::minmax-completion} ensures that all completion times are less than or equal to $t$, \eqref{eq::increasing} ensures that all slices are larger than the overlap $o$, while \eqref{eq::integer} reflects that a frame can be divided at pixel positions only.

Solving the CTM problem in the considered VSN scenario faces three major challenges.
First, sensors may not have sufficient information to formulate the CTM problem, because both the transmission time coefficients and the interest point distribution are unknown before processing a frame.
Second, due to the computational constraints of the sensors and due to the complexity of the CTM problem it may be infeasible to solve even small instances of the CTM problem in the sensors.
Third, even if every sensor could solve the optimization problem, there may be multiple solutions, and deciding which solution to use would require communication
between the sensors, which introduces delay and control traffic overhead.

The following theorem shows that the CTM problem is indeed computationally hard.
\begin{theorem}
The optimization problem \eqref{eq::ip-objective}-\eqref{eq::integer} is NP-hard.
\label{th::np-hardness}
\end{theorem}
\begin{proof}
We prove the NP-hardness via reduction from the \emph{Multiprocessor scheduling} problem, which is known to be NP-hard. Given a set $\mathcal{J}$ of $|\mathcal{J}|=J$ jobs where job $j\in \mathcal{J}$ has length $l_j$ and $M$ number of processors, what is the minimum possible time required to schedule all jobs in $\mathcal{J}$ on $M$ processors such that none overlap. We reduce the problem by setting $S=J$. $N=M$, $C_{s,n}=l_j$ for $s=j=1...S $, $P_{s}=0$, $o=0.5$.

Due to the large $o$ value, it is never beneficial to divide the load to more than one processor, and consequently the solution to \eqref{eq::ip-objective}-\eqref{eq::integer} gives the job assignment with minimum finishing time in the multiprocessor scheduling problem.
\end{proof}

\subsection{Solution Architectures for the CTM Problem}
To address the above challenges, in the rest of the paper we propose and compare three solution architectures for solving the CTM problem.
In all three architectures we address the first challenge by utilizing low complexity interest point distribution prediction~\cite{EDF2016TMC}, but the three architectures
differ in terms of the role of the sensors and of the central entity.
\begin{itemize}
\item
\emph{Distributed operation}: Each sensor $s$ optimizes $(d^i_s,x_s^i)$ in a distributed manner based on information available to it. We introduce the distributed algorithms used by the sensors in Section~\ref{seq::algorithms}.
  \item
\emph{Centralized operation}: A central entity computes the assignment~$({\bf x^i},{\bf d^i})$  periodically  after every $R$ multi-view frames, and sends it to the sensors. The sensors use the same assignment for the subsequent $R$ multi-view frames, i.e., they do not optimize either the assignment or the allocation. We refer to $R$ as the inter-refresh time. We introduce the algorithm used by the central entity in Section~\ref{sec::central-alg}.
\item
  \emph{Coordinated operation}: A central entity computes the assignment~$({\bf x^i},{\bf d^i})$  periodically after every $R$ multi-view frames, and sends it to the sensors. The sensors cannot update the assignment ${\bf d^i}$ but may update the allocation $x_s^i$ using a distributed algorithm. We describe coordinated operation in Section~\ref{sec::coordinated-alg}
\end{itemize}

\section{Distributed Algorithms}
\label{seq::algorithms}
In what follows we propose four distributed algorithms that 
differ in terms of the information available to sensor $s$ before processing frame $i$, denoted by $\Upsilon_{s}^{i}$. 
Node $s$ may obtain the information through measurements (e.g., the transmission and processing time for its frame $i-1$), or through signaling from other processing nodes (e.g., slice sizes of other sensors).

To make the available information $\Upsilon_{s}^{i}$ explicit, we introduce the predicted mean transmission time coefficient $\bar{C}_{s,n}(x_s^i,d_s^i|\Upsilon_{s}^{i})$,
and the predicted mean processing time coefficient $\bar{P}_{s,n}(x_s^i,d_s^i|\Upsilon_{s}^{i})$. 
Based on the information $\Upsilon_{s}^{i}$, the predicted mean completion time of sensor $s$ can be expressed for $n=d_s^i(1)$ as
\begin{eqnarray}
\bar{T}_{s,n}^{i}(x_s^i,d_s^i|\Upsilon_{s}^{i}) &=& \bar{C}_{s, n}(x_s^i,d_s^i|\Upsilon_{s}^{i}) [o + y_{s,1}^{i}] \label{eq::node1-completion-time-expected}\\
& +&  \bar{P}_{s,n}(x_s^i,d_s^i|\Upsilon_{s}^{i}) [y_{s,1}^{i} + \alpha_d\xi_{s,1}^i], \nonumber
\end{eqnarray}
and for the remaining slices, i.e., $n=d_s^i(v)$, $v>1$, we have
\begin{eqnarray}
  \bar{T}_{s,n}^{i} (x_s^i,d_s^i|\Upsilon_{s}^{i})&=& \bar{C}_{s, d^i_s(1)}(x_s^i,d_s^i|\Upsilon_{s}^{i}) [y_{s,1}^{i}+o]  \label{eq::node-completion-time-expected}\\
   &+& \sum_{\nu=2}^{v}{\bar{C}_{s, d^i_s(\nu)}(x_s^i,d_s^i|\Upsilon_{s}^{i}) [y_{s,\nu}^{i}+2o]} \nonumber \\
   &+&  \bar{P}_{s,n}(x_s^i,d_s^i|\Upsilon_{s}^{i}) [y_{s,v}^{i}+\alpha_d\xi_{s,v}^{i}].\nonumber
\end{eqnarray}
Finally, sensor $s$ aims to minimize its predicted completion time
\begin{equation}
	\bar{T}_s^i(x_s^i,d_s^i|\Upsilon_{s}^{i})= \max_{n\in d_s^i} \bar{T}_{s,n}^{i}(x_s^i,d_s^i|\Upsilon_{s}^{i}).
\label{eq::sys-completion-time-expected}
\end{equation}
The times when the sensors can revise their allocations are determined by the \emph{revision opportunity}, which can be either
synchronous or asynchronous.
\begin{definition}{\bf Revision opportunity:}
\label{def::revision-opportunity}
Asynchronous revision allows one sensor $s\in{\cal S}$ to update its allocation upon each frame $i$.
Synchronous revision allows every sensor to update its allocation upon every frame $i$.
\end{definition}
While in a VSN synchronous revision is straightforward to implement, asynchronous revision could, e.g., be implemented by configuring
a static revision order through modulo division of the frame sequence number, letting sensor $s$ revise its allocation at frame sequence numbers $i \mod S = s$.

A basic requirement for the visual sensor network design is to achieve predictable, stable performance.
In the case when the frame contents do not change, stability can be guaranteed if the distributed algorithms reach an equilibrium allocation where the sensors settle. As a result the predicted completion time remains constant. We formally define an equilibrium as follows.

\begin{definition}{\bf Equilibrium:}
An equilibrium is an assignment profile $(d_s^i)_{s\in{\cal S}}$ and $(x_s^i)_{s\in{\cal S}}$ compared to which no sensor $s$
can decrease its predicted completion time by deviating unilaterally, given the information $\Upsilon_{s}^{i}$.
\end{definition}
In the case of perfect information (i.e., if $\Upsilon_{s}^{i}$ contains all transmission
time coefficients, processing time coefficients and interest point distributions), the notion of an equilibrium corresponds to the notion of a Nash equilibrium in game theory~\cite{Monderer96,Pacifici12jsac}.

In the following we analyse the stability of the distributed algorithms.
We consider two scenarios, according to $\Upsilon_{s}^{i}$, the information available at the sensors.
The algorithms are characterized by the \emph{combination} of the \emph{calculation} of the assignment based on $\Upsilon_{s}^{i}$, and by the \emph{revision opportunity}, that is, when changes are performed at the sensors. 
For the stability analysis below we assume that the interest points are evenly spread along the horizontal axis in every frame, i.e. $\xi_{s,v}^{i}$ is proportional to $y_{s,v}^{i} \forall s,v$. We can thus omit the interest point distribtion function from the solution, without altering the resulting allocation.  
For notational convenience we omit the index $i$ whenever the predicted transmission time and processing time coefficients are used.

\subsection{Measurement Only (MO) Information}

We start with considering a system with no signaling between the processing nodes and the sensor nodes, thus,
all parameters need to be estimated by the sensors. We call this the measurement only (\emph{MO}) scenario. Sensor $s$ measures communication and processing times, that is, $t^r_{s,v}-t^b_{s,v}$ and $t^c_{s,v}-t^r_{s,v}$, and estimates the experienced transmission time coefficient $\tilde{C}_{s,n}$ and experienced processing time coefficient $\tilde{P}_{n}$ at processing node $n\in d^i_{s}$ according to \eqref{eq::node1-communication-coeff} and \eqref{eq::node1-processing-coeff}.

Let us consider sensor $s$ and let us derive the optimal offloading for a particular assignment function $d_{s}$, given
$\bar{C}_{s,n}=\tilde{C}_{s,n}$ and $\bar{P}_{n}=\tilde{P}_{n}$.
In order to find the optimal assignment $d_{s}$ and to calculate the optimal allocation ${\bf x}_s$,
we recall a fundamental result from divisible load theory~\cite{BharadwajCC2003}.
\begin{lemma}
The completion time $T_{s}^{i}$ for sensor $s$ is minimized if all processing nodes $n\in d_s$ complete processing at the same time. Furthermore,
if all processing time coefficients are equal then the optimal assignment is in increasing order of the transmission time coefficients $C_{s,n}$ (i.e., use
the node with fastest link first).
\end{lemma}
This result is illustrated in Figure~\ref{fig::completion-illustration} for $N=3$ processing nodes.
The fact that at optimality all used processing nodes $n\in d_s$ complete processing at the same time
allows us to establish a relationship between the optimal slice widths for a particular assignment $d_{s}$ as

\begin{flalign}
\small
& \bar{P}_{d_s(1)} y_{s,1} = \bar{C}_{s,d_s(2)} 2 o + ( \bar{C}_{s,d_s(2)} + \bar{P}_{d_s(2)} ) y_{s,2} \\
& \bar{P}_{d_s(2)} y_{s,2} = \bar{C}_{s,d_s(3)} 2 o + ( \bar{C}_{s,d_s(3)} + \bar{P}_{d_s(3)} ) y_{s,3} \\
& \ldots  \nonumber\\
& \bar{P}_{d_s(V\scalebox{0.66}[1.0]{\( - \)}1)} y_{s,V\scalebox{0.66}[1.0]{\( - \)}1} = \bar{C}_{s,d_s(V)} o + ( \bar{C}_{s,d_s(V)} + \bar{P}_{d_s(V)} ) y_{s,V} \hspace{-0.7mm}. \hspace{-3mm}\label{eq::rec_V-1}
\end{flalign}

Recall that slice $V$ corresponds to the right edge of the original frame, and thus only one overlap region is considered in \eqref{eq::rec_V-1}.
These equations allows formulating the recursive expression for the normalized width of slices $1 \leq v < V-1$,
\begin{equation}
y_{s,v} =\frac{2o\bar{C}_{s,d_s(v+1)}}{\bar{P}_{d_s(v)}} + \frac{\bar{P}_{d_s(v+1)}+\bar{C}_{s,d_s(v+1)}}{\bar{P}_{d_s(v)}} y_{s,v+1},
\label{eq::slice-recursion}
\end{equation}
as well as for slice $v = V-1$
\begin{equation}
y_{s,V-1} =\frac{o\bar{C}_{s,d_s(V)}}{\bar{P}_{d_s(V-1)}} + \frac{\bar{P}_{d_s(V)}+\bar{C}_{s,d_s(V)}}{\bar{P}_{d_s(V-1)}} y_{s,V},
\label{eq::slice-recursion-last}
\end{equation}
which, together with the normalization constraint ${\sum_{v=1}^{V}{y_{s,v}}=1}$ give the optimal allocation vector.

Given the above equations for the optimal slice widths, in the MO scenario each sensor $s$ selects an allocation $(d_s,x_s)$ by calculating the optimal slice widths for all possible assignment functions $d_s$, and then by selecting the allocation leading to the lowest estimated completion time.

As we show next, the optimal slice widths have an interesting property that can be leveraged for the equilibrium analysis of the \emph{MO} scenario.
\begin{lemma}
Given an assignment function $d_s$, the optimal slice widths $y^*_{s,v}$ are insensitive to the scaling to the predicted transmission time coefficients $\bar{C}_{s,n}$ and of the predicted processing time coefficients $\bar{P}_n$ by the same factor $\sigma>0$.
\label{prop::singlenode-insensitive-cut}
\end{lemma}
\begin{proof}
Observe that (\ref{eq::slice-recursion}) is an affine function, and thus due to the normalization constraint the ratio $y_{s,v}/y_{s,v+1}$ does not change as long as the ratios
$\frac{\bar{C}_{s,d_s(v+1)}}{\bar{P}_{d_s(v)}}$ and $\frac{\bar{P}_{d_s(v+1)}+\bar{C}_{s,d_s(v+1)}}{\bar{P}_{d_s(v)}}$
are unchanged.
Since the optimal slice widths are obtained by using the fact that $\sum_{v=1}^{V}{y_{s,v}}=1$, the optimal slice widths $y_{s,v}^*$ are only a function of $y_{s,v}/y_{s,V}$, and thus the result follows.
\end{proof}

\begin{lemma}
Let $d_s^*$ be the assignment function that together with cutpoint location vector $x_s^{*}$ minimizes the completion time for sensor $s$.
Then $d_s^*$ and $x_s^{*}$ are optimal after scaling all predicted transmission time coefficients $\bar{C}_{s,n}$ and all predicted processing time coefficients $P_n$ by the same factor $\sigma>0$.
\label{prop::singlenode-insensitive-assignment}
\end{lemma}
\begin{proof}
Observe that by Lemma~\ref{prop::singlenode-insensitive-cut}, the cutpoint location vector $x_s^{*}$ remains optimal for $d_s^*$ after scaling.
Furthermore, the completion time is a linear function of the transmission and of the processing time coefficients, and thus all completion
times $T_s(x_s,d_s)$ are scaled by $\sigma$. Thus, $\bar{T}_s(x_s^{*},d_s^*)$ remains minimal after scaling.
\end{proof}

We are now ready to prove a sufficient condition for an equilibrium allocation to exist for the \emph{MO} scenario.
\begin{theorem}
Consider a VSN with symmetric transmission time coefficients $C_{s,n}=C_{s^\prime, n}, \forall s,s ^\prime \in \mathcal{S}$. Let us define $d_s^*$ and  $x_s^{*}$, the assignment function and the corresponding cutpoint location vector that are optimal for $\tilde{C}_{s,n}=C_{s,n}$ and $\tilde{P}_{n}=P_{n}$. 
If all the sensors use all processing nodes, i.e., $d_s^*={\cal N}$, then an equilibrium allocation exists under \emph{MO}, and $(d_s^*)_{s\in S}$ and $(x_s^{*})_{s\in S}$ is an equilibrium allocation profile.
\label{theo::eq-existence-mo}
\end{theorem}
\begin{proof}
Observe that for every sensor $s$ the experienced transmission time coefficients $\tilde{C}_{s,n}=S\times C_{s,n}$ and the experienced processing time coefficients $\tilde{P}_{n}=S\times P_{n}$.
By Lemma~\ref{prop::singlenode-insensitive-assignment} the optimal assignment function $d_s^*$ and the optimal cutpoint location vector $x_s^{*}$ are insensitive to scaling and thus they remain optimal for all sensors. Consequently, $(d_s^*, x_s^*)_{s\in{\mathcal{S}}}$ is an equilibrium.
\end{proof}

As a consequence, an equilibrium exists for symmetric systems, and it can be reached if all senors use the same initial assignment function with  $d_s^*={\cal N}$, and calculate the assignment vectors assuming $\tilde{C}_{s,n}=C_{s,n}$ and $\tilde{P}_{n}=P_{n}$.
Given that an equilibrium exists and could easily be reached, it is important to understand whether, in general, the completion time would be minimal in an equilibrium. The following result shows that this is not the case.

\begin{proposition}
\label{prop::mo-equ-may-not-be-optimal}
An equilibrium allocation under the \emph{MO} scenario may not be optimal.
\end{proposition}
\begin{proof}
We prove the proposition through an example. Let $S=2$, $N=2$, $C_{s,n}=1$ and $P_n=5$, $o=0.1$.
Then, in isolation $d_s^*=(1,2)$, $x_s^*=(0,\frac{6.1}{11},1)$, which is an equilibrium according to Theorem~\ref{theo::eq-existence-mo},
with completion time $T^*_s=\unit[6.85]{s}$.
By changing the assignment function of sensor $2$ to $d_s^*=(2,1)$, while maintaining the same cutpoint location vectors $x_s^*=(0,\frac{6.1}{11},1)$ the completion time is $T_s=\unit[6.31]{s}<T^*_s$, thus
$d_s^*=(1,2)$, $x_s^*$ cannot be optimal.
\end{proof}
Thus, while the \emph{MO} scenario requires no signaling, even if sensors would converge to an equilibrium, the performance may not be optimal.

\subsection{Transmission Time (TT) Information}
\label{sec::transmission-time-info}
In the second scenario each sensor can measure its transmission and processing time coefficients as in the \emph{MO} scenario.
Besides, upon completion, every processing node $n$ broadcasts to each sensor $s$ its processing time coefficient $P_{n}$, and the beginning and end of the transmission times
$(t^b_{s^\prime,\rotatedd_{s^\prime}^i(n)},t^r_{s^\prime,\rotatedd_{s^\prime}^i(n)})$ and the corresponding slice widths $y^i_{s^\prime,\rotatedd_{s^\prime}^i(n)}$ for all sensors
that used node $n$, i.e., $s^\prime\in\{s^\prime|\exists v \; \textrm{s.t.}\; d^i_{s^\prime}(v)=n \}$. We refer to this as the transmission time (\emph{TT}) scenario.
Observe that $(t^b_{s^\prime,\rotatedd_{s^\prime}^i(n)}-t^r_{s^\prime,\rotatedd_{s^\prime}^i(n)})$ is a known linear function of $y^i_{s^\prime,\rotatedd_{s^\prime}^i(n)}$ and of the transmission time coefficient $C_{s^\prime,n}$, and thus every sensor $s$ can compute $C_{s^\prime,n}$ for $n\in d^i_{s^\prime}$.

In order to get analytic insight into the problem, let us make the simplifying assumption that the
experienced transmission times do not change as an effect of the sensors' assignments. The assumption holds for example when $C_{s,n}=C_{s^\prime, n}, \forall s,s ^\prime \in \mathcal{S}$.
Under this simplifying assumption we show that an equilibrium allocation exists for the \emph{TT} scenario.
\begin{theorem}
There is an equilibrium allocation $({\bf x^*}, {\bf d^*})$ such that no sensor can decrease its completion time by unilaterally changing its allocation.
\label{theo::eq-existence-tt}
\end{theorem}
\begin{proof}
For an allocation $({\bf x}, {\bf d})$ let us define the vector $\tau({\bf x}, {\bf d})=(T_{s,1},\ldots,T_{s,N})$ of completion
times at the processing nodes sorted in decreasing order,  i.e., $T_{s,1}({\bf x}, {\bf d}) \geq T_{s,2}({\bf x}, {\bf d})$, etc.

Let us now consider that every sensor $s$ chooses a cutpoint location vector $x_s^{1}$ and an assignment vector $d_s^{1}$ that minimizes its completion time
assuming there are no other sensors. We refer to this initial assignment as $({\bf x^{1}},{\bf d^{1}})$.

Let us consider now that given assignment $({\bf x^{i}},{\bf d^{i}})$, $i\geq 1$, a single sensor $s$ revises its assignment and/or allocation to $(x_s^\prime,d_s^\prime)$ and thereby
it minimizes its completion time given the assignments ${\bf d_{-s}^{i}}$ and allocations ${\bf x_{-s}^{i}}$ of the other sensors, i.e.,
\begin{equation}
T_s((x_s^\prime,{\bf x_{-s}^{i}}),(d_s^\prime, {\bf d_{-s}^{i}})) < T_s({\bf x^{i}},{\bf d^{i}}).
\label{eq::completion-time-better reply}
\end{equation}
Let us denote by $({\bf x^{i+1}},{\bf d^{i+1}}) = ((x_s^\prime,{\bf x_{-s}^{i}}),(d_s^\prime, {\bf d_{-s}^{i}}))$ the
resulting assignment profile. Observe that (\ref{eq::completion-time-better reply}) implies that
\begin{equation}
max_{n\in d_s^\prime}T_{s,n}((x_s^\prime,{\bf x_{-s}^{i}}),(d_s^\prime, {\bf d_{-s}^{i}})) < max_{n\in d_s}{T_{s,n}({\bf x^{i}},{\bf d^{i}})}.
\end{equation}
At the same time, for $n\not\in d_s^\prime$ we have $T_{s,n}((x_s^\prime,{\bf x_{-s}^{i}}),(d_s^\prime, {\bf d_{-s}^{i}}))\leq {T_s({\bf x^{i}},{\bf d^{i}})}$.
Thus,
\begin{equation}
\tau({\bf x^{i+1}}, {\bf d^{i+1}}) <_L \tau({\bf x^{i}}, {\bf d^{i}}),
\end{equation}
where $<_L$ stands for \emph{lexicographically smaller}.
Observe that among all vectors $\tau$ of ordered completion times there is a vector that is lexicographically minimal; it is the vector that corresponds to all sensors completing at the same time.
Thus, there is an allocation $({\bf x}, {\bf d})$ compared to which no sensor can decrease its completion time.
\end{proof}

Observe that the proof is based on an \emph{asynchronous} revision opportunity. A consequence of the proof
is that using \emph{asynchronous} revision the sensors can reach an equilibrium in the \emph{TT} scenario.
\begin{corollary}
\label{cor::tt-async-converges}
Assume that sensors follow the asynchronous revision opportunity. Then the sensors' allocations converge to an equilibrium under the \emph{TT} scenario.
\end{corollary}
Thus, using asynchronous revision under the \emph{TT} scenario guarantees convergence to equilibrium.
Unfortunately, this result cannot be extended to the case of \emph{synchronous} revision, as shown by the
following example.

\begin{example}
Let ${\cal S}=\{1,2\}$, ${\cal N}=\{1,2\}$, $C = \left(\begin{smallmatrix} 1 & 2 \\ 2 & 1 \end{smallmatrix}\right)$, $P_{n}=5$, and $o=0.1$.
Let the initial assignment be $d_{1}^1=(1,2)$, $d_{2}^1=(2,1)$ and the allocations $x_1^1=(0,0.6,1)$ and $x_2^1=(0,0.5,1)$.
Then $T^1=6.9$ and the senors will swap their allocations to $x_1^2=(0,0.5,1)$ and $x_2^2=(0,0.6,1)$, which results in $T^2=6.9$.
The next allocation is $x_1^3=(0,0.6,1)=x_1^1$, and $x_2^3=(0,0.5,1)=x_2^1$ thus the sensors will cycle between these two allocations.
\end{example}

To avoid that sensors cycle between two allocations, for the synchronous revision case we introduce \emph{synchronous/S} revision, where sensors
allocate their processing load according to the weighted average $x_s^i=\frac{1}{S}x_s^\prime+\frac{S-1}{S}x_s^{i-1}$.
It is easy to verify that the cycle in the above example can be avoided with this new revision rule
and the sensors would reach the equilibrium $x_1^2=(0,0.55,1)$ and $x_2^2=(0,0.55,1)$ in one step resulting in $T^2=6.3$.

\section{Centralized and Coordinated Algorithms}
\label{sec::centralized}

Theorem~\ref{th::np-hardness} implies that solving the CTM problem for each multi-view frame is infeasible even for moderate instances of the problem. In what follows we propose a heuristic that builds on the storage and computational capability of the central entity to compute near-optimal assignment and allocation profiles for a subset of past multi-view frames off-line, and utilizes a nearest neighbour search to select the assignment profile for the subsequent multi-view frame $i$. We then define coordinated operation.

\subsection{Near-optimal Centralized Algorithm}
\label{sec::central-alg}
There are three observations underlying the proposed heuristic. 
First, due to the non-linearity of the system completion time, the optimization problem~\eqref{eq::ip-objective} is non-trivial to solve even for a fixed assignment profile, hence  we rely on an approximate solution. 
Second, the optimal assignment for a multi-view frame is determined by the interest point distributions of the frames. 
As shown in~\cite{EDF2016TMC}, the interest point distribution $F_s^i$ can be efficiently approximated by using $Q-1$ quantile points 
\begin{eqnarray}
q_s^i(p) = \inf\left\lbrace x\in\left[ 0,\ldots,w\right] | \frac{p}{Q} \leq F_s^i(x) \right\rbrace, 1 \leq p < Q.
\end{eqnarray}
The quantile approximation also allows fast allocation optimization for a single sensor using linear relaxation as well as it supports last value based prediction of the interest point distribution with low prediction error~\cite{EDF2016TMC}. 
Second, in typical surveillance applications the sensors would observe fairly similar scenarios, hence it is reasonable to assume that the number of assignment profiles used would be limited. Motivated by these observations, the proposed heuristic is as follows.

{\bf Off-line approximation (TT/C):} For a past multi-view frame $i$ we use an off-line iterative algorithm for approximating the optimal solution of~\eqref{eq::ip-objective}. The algorithm resembles the asynchronous \emph{TT} algorithm, described in Section~\ref{sec::transmission-time-info}, with the difference that the objective function of each sensor is $T^i({\bf x^i},{\bf d^i})$, instead of $T^i_s({\bf x^i},{\bf d^i})$ as given in Eq.~(\ref{eq::sys-completion-time-expected}), i.e., each sensor aims to minimize the system completion time. As it is centralized, we refer to this algorithm as the \emph{TT/C} algorithm.

{\bf Nearest-neighbor search:} The central entity maintains information about a set $\mathcal{M}$, $|\mathcal{M}|=M$ of past multi-view frames; the information maintained about a multi-view frame $j\in\mathcal{M}$ is (i) the vector ${\bf{q}}^j=(q_1^j,\ldots,q_{S}^j)$ of $S$ quantile vectors of length $Q-1$, where $q_s^j=(q_s^j(1), \ldots, q_s^j(Q-1)$ is the quantile vector for the frame captured by sensor $s$, and (ii) the near-optimal assignment profile $({\bf x^j},{\bf d^j})$ for multi-view frame $j$, computed using the \emph{TT/C} algorithm. 
We use $\mathcal{M}_d \subseteq \mathcal{M}$ to denote the subset of multi-view frames in $\mathcal{M}$ that have the same optimal assignment profile $\bf d$.

For multi-view frame $i$ the central entity uses the last value predictor for predicting the quantile vector $\tilde{\bf{q}}^i$, i.e., $\tilde{\bf{q}}^i={\bf{q}}^{i-1}$. It then finds the set ${\cal L}, |\mathcal{L}|=L,$ of multi-view frames in ${\cal M}$ that have nearest quantile vectors in terms of
$e^{ij}=
\sum_{s=1}^S\sum_{p=1}^{Q-1} (\tilde{q}_s^i(p)-q_s^j(p))^2$.
Finally, given the set ${\cal L}$, the central entity computes the completion time $t^{ij}$ for the predicted quantile vector $\bf{q}^{i}$ and the optimal assignment profiles $( {\bf x^j} {\bf d^j} )$, $j\in{\cal L}$ and it selects the assignment profile that results in the lowest completion time. The pseudo-code of the algorithm is shown in Algorithm \ref{alg::optimal-d}. \emph{AddToHeap}$(\mathcal{L},j)$ adds entry $j$ to $\mathcal{L}$, such that the entries with the smallest $e_{ij}$ values are kept. The sorted list of the $L$ selected entries can be generated in $\mathcal{O}(MSQ \log L)$ time, using a heap data structure of maximum size $L$.

An alternative to the nearest-neighbour search could be to use space partitioning techniques, such as the $k$-d tree~\cite{Bentley:1975:MBS:361002.361007}. Unfortunately, space partitioning is not efficient in our case due to the high dimension of the quantile vector, hence our choice of nearest neighbor search. 

\subsection{Coordinated operation}
\label{sec::coordinated-alg}
Under coordinated operation the central entity uses the algorithm described in Section \ref{sec::central-alg} for every $i \mod R = 0$ frames,
and provides the near-optimal processing node assignment and slice allocation $(\bf{d^i},\bf{x^i})$ to the sensors. The sensors keep the assignment profile $\bf{d^i}$ for the next $R$ frames, and update only $\bf{x^{i}}$, following one of the distributed algorithms defined in Section~\ref{seq::algorithms}.

\begin{algorithm}[t]
\caption{Assignment profile selection algorithm.}
\label{alg::optimal-d}
\begin{algorithmic}[1]
\small
\REQUIRE $\tilde{\bf{q}}^i$, $\mathcal{M}$, $L$
\ENSURE $({\bf x^i},{\bf d^i})$

\STATE{$\mathcal{L}=\{\}$}

\WHILE{$|\mathcal{L}| < L$}
    \STATE $e^{ij*} = \min_{j \in  {\mathcal{M}\setminus\mathcal{L}}} e^{ij}$
    \STATE $j^* = \arg\min_{j \in  {\mathcal{M}\setminus\mathcal{L}}} e^{ij}$
    \STATE AddToHeap$(\mathcal{L},j^*)$
\ENDWHILE
\STATE $j^* = \arg\min_{j \in  \mathcal{L}} t^{ij}$
\STATE $({\bf x^i},{\bf d^i})= ({\bf x^{j*}},{\bf d^{j*}})$

\RETURN{$({\bf x^i},{\bf d^i})$}

\end{algorithmic}

\end{algorithm}

\section{Numerical Results}
\label{sec::numerical}

\begin{figure}[t]
\includegraphics[width=\columnwidth]{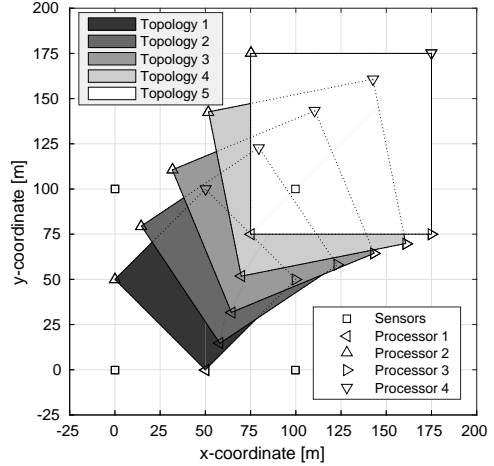}
\caption{The position of the nodes for the five considered topologies. Processing nodes are positioned at the corners of the shaded squares, with each shade representing a different topology. The positions of the sensor nodes remain fixed.}
\label{fig::topologies}
\vspace{-6mm}
\end{figure}

In the following we use simulations based on synthetic data, as well as based on video traces to evaluate the performance of the algorithms.
We consider five VSN topologies for the numerical evaluation, each with four sensor nodes and four processing nodes.
In each of the five topologies, the sensor nodes are placed at the corners of a square with side-length $\unit[100]{m}$.
For Topology 1, the processing nodes are placed at the mid-points of the square's sides, forming a second, smaller square.
For Topology 5, the processing nodes form a square with side-length $\unit[100]{m}$, shifted $\unit[75]{m}$ horizontally and vertically relative to the square formed by the sensor nodes.
Topologies 2-4 are intermediate steps in the transformation from Topology 1 to Topology 5.
In each step of the transformation, the square formed by the processing nodes is shifted $\unit[18.75]{m}$ horizontally and vertically, rotated $\unit[-11.25]{^\circ}$ around its center and increased in size by $\unit[7.32]{m}$ on all sides compared to the square formed in Topology 1.
Figure~\ref{fig::topologies} shows the positions of sensor and processing nodes in the five topologies.

The five topologies provide different challenges for the distributed algorithms.
In Topology 1, an optimal allocation is one where all sensor nodes allocate loads mainly to the processing nodes closest to them and possibly a smaller load to one of the remaining nodes.
Because of the symmetry of Topology 1, there are always at least two allocations where the system achieves the optimal completion time.
Therefore, the main challenge of the distributed algorithms is to converge to the same optimal allocation.
In Topology 5, all the sensors prefer the centrally located processing node 1, and the challenge is to move away from this, globally not optimal allocation.
The intermediary topologies can provide insight to how the performance of the system changes as the topology shifts from one extreme to the other.

We compute the transmission time coefficients $C_{s,n}$ based on the Shannon capacity with bandwidth $\unit[20]{MHz}$, noise-level $\unit[-70]{dBm}$, and free-space path loss assuming a carrier frequency of $\unit[2.4]{GHz}$.
Recall that we assume a MAC protocol that provides airtime fairness, scaling the actual transmission times by the number of transmitting nodes.
For each Topology, the processing time coefficients of all processing nodes are set to the same value as the transmission time coefficient of the slowest link, scaled by the number of sensors in the Topology, i.e. $P_n = S \cdot \min_{s,m} C_{s,m}, \forall n$.
This ensures that both the transmission and processing phase have a significant contribution to the total completion time, and that the ratio between transmission and processing time is comparable across all topologies.

We use BRISK~\cite{Leuten2011} for detecting local visual features with a filter width of up to $84$ pixels (i.e., $o=0.06$), and select the top $400$ interest points to compute the interest point distribution of each frame.
As the sensor nodes can not know the distribution of interest points in frame $i$ before frame $i$ has been processed, they assume that frame $i$ has the same interest point distribution as frame $i-1$.
This corresponds to the last value predictor used in~\cite{EDF2014DCOSS}, which was shown to provide a good trade-off between prediction accuracy and computational complexity.

For each topology we evaluate the completion time of the system under the \emph{MO} and \emph{TT} scenarios with both asynchronous and synchronous/S revisions, i.e, four algorithms, with and without coordination. 
For obtaining the initial estimate of the transmission and processing time coefficients, the sensors use a bootstrap cutpoint location vector $x_s^i$ in which $y_{s,n}^i = \max(o,1/N)$.
Results are evaluated over 500 frames. 

\subsection{Evaluation with synthetic data}
\label{sec::numerical-synthetic}
We first evaluate the algorithms on a sequence of multi-view frames in which every frame has a uniform interest point distribution.
The uniform distribution has been shown to be a good approximation of the average distribution of interest points in frames~\cite{KhanTMM15,KhanDSP13}.
This configuration allows us to observe the convergence properties of the algorithms and how convergence affects the completion time.
We refer to the asynchronous and synchronous \emph{MO} and \emph{TT} algorithms as \emph{MO/A}, \emph{MO/S}, \emph{TT/A}, and \emph{TT/S} respectively. 

\begin{figure}[t]
\includegraphics[width=\columnwidth]{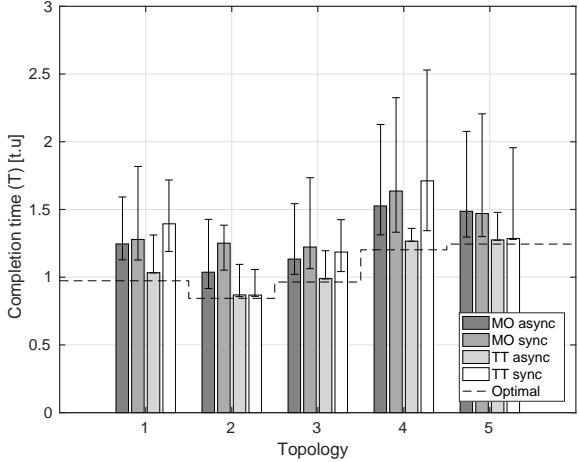}
\caption{The mean, maximum and minimum completion times achieved by the different algorithms for each of the considered topologies. Synthetic data.}
\label{fig::barsummary}
\vspace{-6mm}
\end{figure}

Figure~\ref{fig::barsummary} shows the mean completion times obtained with the four distributed algorithms for the five topologies.
The error bars show the minimum and the maximum completion times, i.e., the variation around the mean, and the dashed line shows the optimal system completion time for each topology.
The figure shows that only the \emph{TT/A} algorithm manages to achieve an average completion time close to the optimal for all topologies; whether or not the performance of the \emph{TT/S} algorithm is close to optimal depends very much on the topology. 
The reason is that asynchronous revisions and the information the sensors receive about each other's allocations and assignments in the \emph{TT} scenario deters them from choosing assignment profiles that would lead to very poor performance. 
Similarily, we can observe that \emph{MO/S} regularly performs worse than \emph{MO/A}, again due to simultaneous revisions of the assignments, which cause large completion time fluctuations.
Finally, we note that the largest range of completion times is obtained for Topologies 4 and 5,
which is due to that in these two topologies sensors compete for the use of processing node 1,
which is closest to all sensors. We can thus conclude that an asymmetric placement of processing nodes is detrimental to system performance in general.
Topology 4 results in large variation in completion times, and therefore we consider this topology for detailed evaluation.
Similar results were observed for other topologies.

\begin{figure}[t]
\includegraphics[width=\columnwidth]{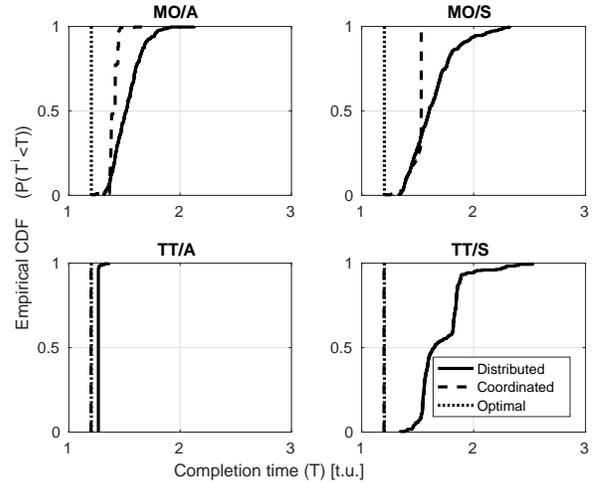}
\caption{Cumulative distribution function of the completion time for Topology 4 with Synthetic data.}
\label{fig::cdf_top3}
\vspace{-6mm}
\end{figure}

Figure~\ref{fig::cdf_top3} shows the cumulative distribution functions (CDFs) of the system completion time of the four algorithms for Topology 4, and allows us to study the convergence properties of the algorithms.
The solid lines show the CDF of the completion times of the distributed algorithms.
To interpret the figure, observe that an algorithm will produce a constant completion time after it reaches an equilibrium, hence a step function-like CDF.
We can see that, as stated in Corollary~\ref{cor::tt-async-converges}, the distributed \emph{TT/A} algorithm converges, while the other algorithms fail to do so.
The reason for the non-convergence of the distributed \emph{MO} algorithms is the mis-prediction of the processing time coefficients based on the processing time coefficients measured for the previous frame:
when sensor $s$ updates its slice sizes, its estimates $\tilde{P}_{s,n}$ of the processing time coefficients are correct only if the proportion of the slice size from $s$ to the total size of the slices from all sensors remains constant.
As the estimation of the processing time coefficients changes over time, eventually a frame $i$ is reached where the sensors will change the assignment function $d_s^i$, and changing the assignment function has a large impact on the experienced processing time coefficients of all sensors in the system, preventing convergence.
Similarly, the non-convergence of the distributed \emph{TT/S} algorithm is due to that the sensors update their allocations based on allocations and assignments observed for the previous frame. 
If any of these changes between two subsequent frames, the change is likely to cause other sensors to update their allocations and prevents the sensots from reaching a stable allocation. To summarize, without coordination only the \emph{TT/A} algorithm provides stable and low average completion times. Therefore, we now investigate the potential benefits of coordination.

The dashed lines in Figure~\ref{fig::cdf_top3} show the completion time CDFs of the four algorithms when central coordination is used. 
Since the interest point distribution is the same in all frames, the optimal assignment profile is constant, hence we let the coordinator provide the optimal assignemnt profile at frame $1$ and we then allow the sensors to update their allocations in the subsequent frames. 
We can observe that with coordination both \emph{TT} algorithms remain in the optimal allocation, as none of the sensor nodes can decrease its completion time compared to the one provided by the central coordinator.
Interestingly, the same does not hold for \emph{MO}. 
Under the \emph{MO} scenario, the sensors deviate from the allocation provided by the central coordinator, and achieve an average completion time that is higher than the optimal.
To summarize, central coordination improves the stability for all algorithms. 

\subsection{Video trace based evaulation}

We now turn to the evaluation of the algorithms using a multi-camera surveillance video trace called \emph{Parking lot}, which is a surveillance video data set proposed for the evaluation of algorithms for tracking humans~\cite{khanshahTPA2009}.
The data set consists of the video traces captured in a parking lot by four surveillance cameras at a resolution of $720 \times 480$ pixels and frame rates of $\unit[30]{fps}$, showing $9$ people moving around.
We do not perform background subtraction on the traces prior to interest point detection, and thus there are a number of interest points belonging to the background that do not change their locations.
The cameras are approximately located at the corners of a square and are facing the center of the square, similar to the positions of the sensors in Topologies 1-5.
As a baseline for comparison for the algorithms we use the \emph{TT/C} algorithm
to compute a near-optimal solution based on the interest point distribution of the current frame. We refer to this baseline as the \emph{Oracle}.

\begin{figure}[t]
\includegraphics[width=\columnwidth]{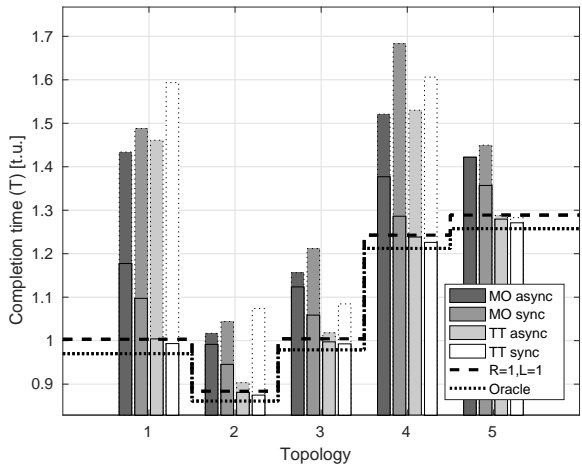}
\caption{Completion time of the algorithms for each of the five topologies. Solid borders mark algorithms with coordination, dotted borders mark algorithms without coordination. \emph{Parking lot} data set, $R=16$.}
\label{fig::barplots}
\vspace{-6mm}
\end{figure}

Figure~\ref{fig::barplots} shows the completion time for the four algorithms and five topologies for the \emph{Parking lot} data set.
Solid borders show the completion time for a coordinated system with a refresh interval of $R=16$, and dotted borders show the completion time without coordination.
The figure shows that coordination can provide significant completion time reductions, especially for synchronous/S algorithms where simultaneous changes of the assignment vector can result in poor assignment profiles without coordination.
It is interesting to observe that with coordination the algorithms using synchronous/S revision opportunities achieve lower completion times than their asynchronous versions.
While without coordination \emph{TT/S} performs poorly, with coordination  it achieves completion times close to those of the \emph{Oracle}. 
Finally, we note that the relative performance of the algorithms is similar for all the topologies. 

\begin{figure}[t]
\includegraphics[width=\columnwidth]{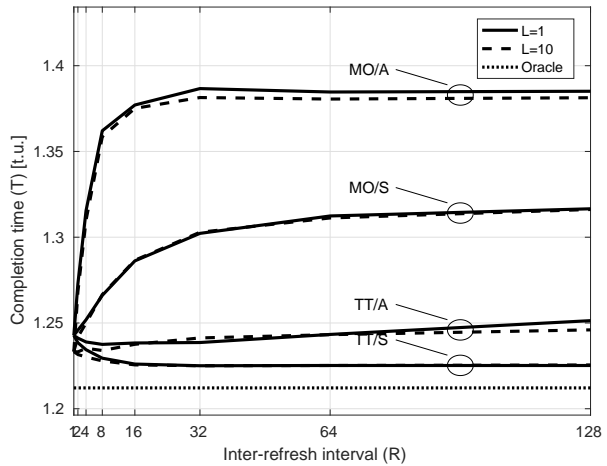}
\caption{Completion time of the algorithms for increasing inter-refresh interval with Topology 4. \emph{Parking lot} data set.}
\label{fig::T-mse-vs-R}
\vspace{-6mm}
\end{figure}

Since coordination requires both computation at the coordinator and signaling to the sensors, we now evaluate how the inter-refresh interval $R$ and the number of evaluated assignment profiles $L$ affect the achievable performance of the coordinated operation.
Figure~\ref{fig::T-mse-vs-R} shows the completion time of the four algorithms as a function of the inter-refresh interval $R$ and for $L=1$ and $L=10$, for Topology 4.
We again show results for Topology 4 because of the large difference in completion times observed in Figure~\ref{fig::barplots}.
Interestingly, the figure shows that increasing the inter-refresh interval from $R=1$ has opposite effects under the \emph{MO} and the \emph{TT} scenarios.
Under the \emph{MO} scenario the opportunity given to the sensors to update the allocation profiles results in increasing completion times, consistent with our observations in Section~\ref{sec::numerical-synthetic}. 
Unlike for \emph{MO}, under the \emph{TT} scenario the completion times decrease as the inter-refresh interval $R$ increases. While this is seemingly counter-intuitive, observe that the allocation provided by the coordinator may be suboptimal, as it is computed usig a quantile-based approximation and nearest-neighbor search, and this makes it possible for the sensors to improve the allocation profile through subsequent iterations. 

It is also interesting to note that with coordination the algorithms using synchronous/S revision opportunity achieve consistently lower completion times than the corresponding asynchronous algorithm. The reason is that under coordination the sensors are not allowed to change the assignment profile, which avoids very poor assignment profiles to be chosen, but simultaneous updates allow the sensors to improve the allocation profile faster.
It is also important to note that the completion time does not change significantly as the inter-refresh interval $R$ increases from $16$ to $64$, which allows us to achieve consistently low completion time with infrequent coordination.

Finally, comparing the results for $L=1$ and $L=10$ we can observe that evaluating more dictionary entries has little impact on the completion time, especially for large inter-refresh intervals. Thus one can achieve low completion times with very low overhead by providing low frequency coordination based on a simple dictionary lookup.

\begin{figure}[t]
\includegraphics[width=\columnwidth]{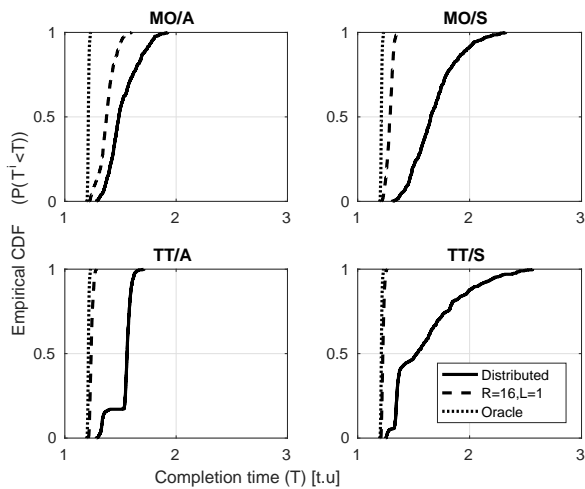}
\caption{Cumulative distribution function of the completion time for Topology 4. \emph{Parking lot} data set, $R=16$, $L=1$.}
\label{fig::cdf_top3_trace}
\vspace{-6mm}
\end{figure}

Figure~\ref{fig::cdf_top3_trace} shows the CDF of the completion times of the four algorithms for Topology 4, with and without coordination. 
Comparing Figure~\ref{fig::cdf_top3} and Figure~\ref{fig::cdf_top3_trace} we see that without coordination the CDFs of the completion times have similar shapes for the synthetic data and for the \emph{Parking lot} data set. In particular, the CDFs show that there is a significant tail probability, i.e., a non-negligible probability that the completion time significantly exceeds the mean completion time. 
It is therefore important to note that coordination not only decreases the mean completion time, but it also reduces the probability of encountering completion times that are significantly higher than the average completion time. 
This is especially apparent for the \emph{TT} scenario, under which coordination results in an almost deterministic distribution, i.e., almost constant completion times.

\section{Conclusion and Future Work}
\label{sec::conclusion}
In this paper we considered the problem of minimizing the completion time of distributed feature extraction in visual sensor networks consisting of several camera sensors and image processing nodes.
We proposed distributed solutions, where each camera sensor decides locally the set of processing nodes to be used, the schedule of the data transmission and the size of the frame slices.
We defined four algorithms for the distributed allocation of processing load that differ in terms of the information available to the sensors, and in the revision opportunity used. We extended the distributed solutions by the support of a central coordinator.
We evaluated the algorithms using simulations based on both synthetic data, and on video traces.
Our results show that, independently from the topology considered, fully distributed algorithms require both asynchronous revisions and accurate information on transmission and processing times to achieve completion times close to the optimal.
The support of the central coordinator gives more stable performance, though it may lead to higher average completion times for the static synthetic traffic.
Results using the video trace show that central coordination provides a decreased completion time for all algorithms, even when coordination is provided infrequently. Therefore, we propose the combination of distributed allocation of processing tasks with limited central coordination to provide good visual analysis performance in multi-camera sensor networks with small signalling overhead.

\balance
\begin{footnotesize}
\bibliographystyle{IEEEtran}
\bibliography{refs.bib}

\begin{thebibliography}{10}
\providecommand{\url}[1]{#1}
\csname url@samestyle\endcsname
\providecommand{\newblock}{\relax}
\providecommand{\bibinfo}[2]{#2}
\providecommand{\BIBentrySTDinterwordspacing}{\spaceskip=0pt\relax}
\providecommand{\BIBentryALTinterwordstretchfactor}{4}
\providecommand{\BIBentryALTinterwordspacing}{\spaceskip=\fontdimen2\font plus
\BIBentryALTinterwordstretchfactor\fontdimen3\font minus
  \fontdimen4\font\relax}
\providecommand{\BIBforeignlanguage}[2]{{%
\expandafter\ifx\csname l@#1\endcsname\relax
\typeout{** WARNING: IEEEtran.bst: No hyphenation pattern has been}%
\typeout{** loaded for the language `#1'. Using the pattern for}%
\typeout{** the default language instead.}%
\else
\language=\csname l@#1\endcsname
\fi
#2}}
\providecommand{\BIBdecl}{\relax}
\BIBdecl

\bibitem{Muller2011}
K.~Muller, P.~Merkle, and T.~Wiegand, ``3-{D} video representation using depth
  maps,'' \emph{Proc. of the IEEE}, vol.~99, no.~4, pp. 643--656, 2011.

\bibitem{rana3dtv14}
P.~Rana, J.~Taghia, and M.~Flierl, ``Statistical methods for inter-viewdepth
  enhancement,'' in \emph{3DTV-Conf.: The True Vision - Capture, Transmission
  and Display of 3D Video (3DTV-CON)}, 2014.

\bibitem{liemPR11}
M.~Liem and D.~M. Gavrila, ``Multi-person localization and track assignment in
  overlapping camera views,'' in \emph{Pattern Recognition}, ser. Lecture Notes
  in Comput. Sci., 2011, pp. 173--183.

\bibitem{Zhou2009CVIU}
H.~Zhou, Y.~Yuan, and C.~Shi, ``Object tracking using {SIFT} features and mean
  shift,'' \emph{Comput. Vision and Image Understanding}, vol. 113, no.~3, pp.
  345--352, 2009.

\bibitem{Ayazoglu:2011:DSC:2355573.2356391}
M.~Ayazoglu, B.~Li, C.~Dicle, M.~Sznaier, and O.~I. Camps, ``Dynamic
  subspace-based coordinated multicamera tracking,'' in \emph{Proc. of IEEE
  Intl. Conf. on Comput. Vision (ICCV)}, 2011, pp. 2462--2469.

\bibitem{helmerICRA10}
S.~Helmer and D.~Lowe, ``Using stereo for object recognition,'' in \emph{IEEE
  Int. Conf. on Robotics and Automation (ICRA)}, May 2010.

\bibitem{naikal2010towards}
N.~Naikal, A.~Y. Yang, and S.~S. Sastry, ``Towards an efficient distributed
  object recognition system in wireless smart camera networks,'' in \emph{Proc.
  of IEEE Conf. on Inform. Fusion (FUSION)}, 2010, pp. 1--8.

\bibitem{Muller:2005:RDE:2322564.2323726}
K.~Muller, A.~Smolic, M.~Drose, P.~Voigt, and T.~Wiegand, ``3-{D}
  reconstruction of a dynamic environment with a fully calibrated background
  for traffic scenes,'' \emph{IEEE Trans. Circuits Syst. Video Technol},
  vol.~15, no.~4, pp. 538--549, Apr. 2005.

\bibitem{Duan2012}
L.-Y. Duan, X.~Liu, J.~Chen, T.~Huang, and W.~Gao, ``Optimizing {JPEG}
  quantization table for low bit rate mobile visual search,'' in \emph{Proc. of
  IEEE Visual Commun. and Image Process. Conf. (VCIP)}, 2012.

\bibitem{KhanTMM15}
G.~D\'{a}n, M.~A. Khan, and V.~Fodor, ``Characterization of {SURF} and {BRISK}
  interest point distribution for distributed feature extraction in visual
  sensor networks,'' \emph{IEEE Trans. Multimedia}, vol.~17, no.~5, May 2015.

\bibitem{Rosten2010}
E.~Rosten, R.~Porter, and T.~Drummond, ``Faster and better: A machine learning
  approach to corner detection,'' \emph{IEEE Trans. Pattern Anal. Mach.
  Intell.}, vol.~32, no.~1, pp. 105--119, 2010.

\bibitem{Leuten2011}
S.~Leutenegger, M.~Chli, and R.~Siegwart, ``{BRISK}: Binary robust invariant
  scalable keypoints,'' in \emph{Proc. of IEEE Int. Conf. on Comput. Vision
  (ICCV)}, 2011.

\bibitem{Chao2013}
J.~Chao, H.~Chen, and E.~Steinbach, ``On the design of a novel {JPEG}
  quantization table for improved feature detection performance,'' in
  \emph{Proc. of {IEEE} Int. Conf. on Image Process. ({ICIP})}, 2013.

\bibitem{Chandrasekhar2010}
V.~R. Chandrasekhar, S.~S. Tsai, G.~Takacs, D.~M. Chen, N.-M. Cheung,
  Y.~Reznik, R.~Vedantham, R.~Grzeszczuk, and B.~Girod, ``Low latency image
  retrieval with progressive transmission of {CHoG} descriptors,'' in
  \emph{Proc. of the ACM Multimedia Workshop on Mobile Cloud Media Computing},
  2010.

\bibitem{Jegou2011}
H.~Jegou, M.~Douze, and C.~Schmid, ``Product quantization for nearest neighbor
  search,'' \emph{IEEE Trans. Pattern Anal. Mach. Intell}, vol.~33, no.~1, pp.
  117--128, 2011.

\bibitem{Redondi2013}
A.~Redondi, L.~Baroffio, J.~Ascenso, M.~Cesana, and M.~Tagliasacchi,
  ``Rate-accuracy optimization of binary descriptors,'' in \emph{Proc. of IEEE
  Int. Conf. on Image Process. (ICIP)}, 2013.

\bibitem{Ta2009}
D.-N. Ta, W.-C. Chen, N.~Gelfand, and K.~Pulli, ``{SURFT}rac: Efficient
  tracking and continuous object recognition using local feature descriptors,''
  in \emph{{IEEE} Conf. on Comput. Vision and Pattern Recognition (CVPR)},
  2009.

\bibitem{SullivanTCS12}
G.~Sullivan, J.~Ohm, W.-J. Han, and T.~Wiegand, ``Overview of the high
  efficiency video coding ({HEVC}) standard,'' \emph{IEEE Trans. Circuits Syst.
  Video Technol.}, vol.~22, no.~12, pp. 1649--1668, 2012.

\bibitem{BaraffioICIP13}
L.~Baroffio, M.~Cesana, A.~Redondi, S.~Tubaro, and M.~Tagliasacchi, ``Coding
  video sequences of visual features,'' in \emph{Proc. of IEEE Int. Conf. on
  Image Process. (ICIP)}, 2013.

\bibitem{RedondiDSP13}
A.~Redondi, L.~Baroffio, A.~Canclini, M.~Cesana, and M.~Tagliasacchi, ``A
  visual sensor network for object recognition: Testbed realization,'' in
  \emph{Proc. of Int. Conf. on Digital Signal Process. (DSP)}, 2013.

\bibitem{EDF2014DCOSS}
E.~Eriksson, G.~D\'{a}n, and V.~Fodor, ``Real-time distributed visual feature
  extraction from video in sensor networks,'' in \emph{Proc. of IEEE Int. Conf.
  on Distributed Computing in Sensor Syst. (DCOSS)}, 2014.

\bibitem{BaroffioCCRTDEFAM2014ICIP}
L.~Baroffio, A.~Canclini, M.~Cesana, A.~Redondi, M.~Tagliasacchi, G.~D\'{a}n,
  E.~Eriksson, V.~Fodor, J.~Ascenso, and P.~Monteiro, ``Enabling visual
  analaysis in wireless sensor networks,'' in \emph{Proc. of IEEE Intl. Conf.
  on Image Procesing (ICIP), Show and Tell}, October 2014.

\bibitem{EDF2016TMC}
E.~Eriksson, G.~D{\'a}n, and V.~Fodor, ``Predictive distributed visual analysis
  for video in wireless sensor networks,'' \emph{IEEE Trans. Mobile Comput},
  vol.~15, no.~7, pp. 1743--1756, 2016.

\bibitem{redondi2015cooperative}
A.~Redondi, M.~Cesana, M.~Tagliasacchi, I.~Filippini, G.~D{\'a}n, and V.~Fodor,
  ``Cooperative image analysis in visual sensor networks,'' \emph{Ad Hoc
  Networks}, vol.~28, pp. 38--51, 2015.

\bibitem{redondi2015mathematical}
A.~Redondi, M.~Cesana, L.~Baroffio, and M.~Tagliasacchi, ``A mathematical
  programming approach to task offloading in visual sensor networks,'' in
  \emph{Proc. of IEEE 81st Veh. Technology Conf. (VTC Spring)}, 2015, pp. 1--5.

\bibitem{BharadwajCC2003}
V.~Bharadwaj, D.~Ghose, and T.~Robertazzi, ``Divisible load theory: A new
  paradigm for load scheduling in distributed systems,'' \emph{Cluster
  Computing}, vol.~6, no.~1, pp. 7--17, 2003.

\bibitem{mequanint2006wireless}
M.~Moges and T.~G. Robertazzi, ``Wireless sensor networks: scheduling for
  measurement and data reporting,'' \emph{IEEE Trans. Aerosp. Electron. Syst.},
  vol.~42, no.~1, pp. 327--340, 2006.

\bibitem{li2007sensing}
X.~Li, X.~Liu, and H.~Kang, ``Sensing workload scheduling in sensor networks
  using divisible load theory,'' in \emph{Proc. of IEEE Global Telecommun.
  Conference, (GLOBECOM)}, 2007, pp. 785--789.

\bibitem{BharadwajTPDS1994}
V.~Bharadwaj, D.~Ghose, and V.~Mani, ``Optimal sequencing and arrangement in
  distributed single-level tree networks with communication delays,''
  \emph{IEEE Trans. Parallel Distrib. Syst}, vol.~5, no.~9, pp. 968--976, 1994.

\bibitem{BharadwajTPDS2000}
B.~Veeravalli, X.~Li, and C.-C. Ko, ``On the influence of start-up costs in
  scheduling divisible loads on bus networks,'' \emph{IEEE Trans. Parallel
  Distrib. Syst.}, vol.~11, no.~12, pp. 1288--1305, 2000.

\bibitem{EDF2015NW}
E.~Eriksson, G.~D\'{a}n, and V.~Fodor, ``Algorithms for distributed feature
  extraction in multi-camera visual sensor networks,'' in \emph{Proc. of
  IFIP/TC6 Networking}, 2015.

\bibitem{EPD2015ICMEW}
E.~Eriksson, V.~Pacifici, and G.~D\'{a}n, ``Efficient distribution of visual
  processing tasks in multi-camera visual sensor networks,'' in \emph{Proc. of
  IEEE Int. Conf. on Multimedia \& Expo Workshops (ICMEW)}, 2015.

\bibitem{Monderer96}
D.~Monderer and L.~Shapley, ``Potential games,'' \emph{Games and Econ.
  Behavior}, vol.~14, pp. 124--143, 1996.

\bibitem{Monderer96b}
------, ``Fictitious play property for games with identical interests,''
  \emph{J. of Econ. Theory}, vol.~68, pp. 258--265, 1996.

\bibitem{Pacifici12jsac}
V.~Pacifici and G.~D\'{a}n, ``Convergence in player-specific graphical resource
  allocation games,'' \emph{IEEE J. Sel. Areas Commun. (JSAC)}, vol.~30,
  no.~11, pp. 2190--2199, 2012.

\bibitem{Berger07}
U.~Berger, ``Brown's original fictitious play,'' \emph{J. of Econ. Theory},
  vol. 135, no.~1, pp. 572--578, 2007.

\bibitem{Foster06}
D.~P. Foster and H.~Young, ``Regret testing: Learning to play nash equilibrium
  without knowing you have an opponent,'' \emph{Theoretical Econ.}, vol.~1, pp.
  241--367, 2006.

\bibitem{Cigler2011}
L.~Cigler and B.~Faltings, ``Reaching correlated equilibria through multi-agent
  learning,'' in \emph{Proc. of Int. Conf. on Autonomous Agents and Multiagent
  Syst. (AAMAS)}, May 2011, pp. 509--516.

\bibitem{Pradelski2012}
B.~Pradelsky and H.~Young, ``Learning efficient nash equilibria in distributed
  systems,'' \emph{Games and Econ. Behavior}, vol.~75, pp. 882--897, 2012.

\bibitem{KhanDSP13}
M.~A. Khan, G.~D\'{a}n, and V.~Fodor, ``Characterization of {SURF} interest
  point distribution for visual processing in sensor networks,'' in \emph{Proc.
  of Int. Conf. on Digital Signal Process. (DSP)}, 2013.

\bibitem{Bay2008}
H.~Bay, A.~Ess, T.~Tuytelaars, and L.~V. Gool, ``Speeded-up robust features
  ({SURF}),'' \emph{Comput. Vision and Image Understanding}, vol. 110, no.~3,
  pp. 346 -- 359, 2008.

\bibitem{Calonder2010}
M.~Calonder, V.~Lepetit, C.~Strecha, and P.~Fua, ``{BRIEF}: Binary robust
  independent elementary features,'' in \emph{Proc. of European Conf. on
  Comput. Vision (ECCV)}, 2010.

\bibitem{LacanWIOPT2006}
J.~Lacan and T.~Perennou, ``Evaluation of error control mechanisms for 802.11b
  multicast transmissions,'' in \emph{Proc. of Int. Symp. on Modeling and
  Optimization in Mobile, Ad Hoc and Wireless Networks (WiOpt)}, 2006.

\bibitem{HarwellVTC2004}
J.~Hartwell and A.~Fapojuwo, ``Modeling and characterization of frame loss
  process in {IEEE} 802.11 wireless local area networks,'' in \emph{Proc. of
  IEEE Veh. Technology Conf. (VTC-Fall)}, 2004.

\bibitem{GuhaTVT2008}
R.~Guha and S.~Sarkar, ``Characterizing temporal {SNR} variation in 802.11
  networks,'' \emph{IEEE Trans. Veh. Technol}, vol.~57, no.~4, pp. 2002--2013,
  2008.

\bibitem{Petrova2006wcnc}
M.~Petrova, J.~Riihijarvi, P.~Mahonen, and S.~Labella, ``Performance study of
  {IEEE} 802.15.4 using measurements and simulations,'' in \emph{Proc. of IEEE
  Wireless Commun. and Networking Conf. (WCNC)}, 2006.

\bibitem{Joshi2008TMC}
T.~Joshi, A.~Mukherjee, Y.~Younghwan, and D.~Agrawal, ``Airtime fairness for
  {IEEE} 802.11 multirate networks,'' \emph{IEEE Trans. on Mobile Comput.},
  Apr. 2008.

\bibitem{Kostuch2009Wmnc}
A.~Kostuch, K.~Gierlowski, and J.~Wozniak, ``Performance analysis of multicast
  video streaming in ieee 802.11 b/g/n testbed environment,'' in \emph{Wireless
  and Mobile Networking}, ser. IFIP Advances in Inform. and Commun. Technology,
  J.~Wozniak, J.~Konorski, R.~Katulski, and A.~Pach, Eds.\hskip 1em plus 0.5em
  minus 0.4em\relax Springer, 2009, vol. 308, pp. 92--105.

\bibitem{Bentley:1975:MBS:361002.361007}
J.~L. Bentley, ``Multidimensional binary search trees used for associative
  searching,'' \emph{Commun. ACM}, vol.~18, no.~9, pp. 509--517, Sep. 1975.

\bibitem{khanshahTPA2009}
M.~S. Khan and M.~Shah, ``Tracking multiple occluding people by localizing on
  multiple scene planes,'' \emph{IEEE Trans. Pattern Anal. Mach. Intell},
  vol.~31, no.~3, pp. 505--519, 2009.

\end{thebibliography}
\end{footnotesize}

\pagebreak
\end{document}